\documentclass[conference]{IEEEtran}

\usepackage{times}

\usepackage[numbers]{natbib}
\usepackage{multicol}
\usepackage[hidelinks,bookmarks=true]{hyperref}

\usepackage{amsmath,amssymb,amsthm,stmaryrd,mathtools}
\usepackage{booktabs}
\usepackage{xcolor}
\usepackage{subcaption}
\usepackage[noabbrev]{cleveref}

\crefname{equation}{}{}

\newtheorem{theorem}{Theorem}
\newtheorem{assumption}{Assumption}

\begin{document}

\title{Traversing Supervisor Problem: An Approximately Optimal Approach to Multi-Robot Assistance}


\author{\authorblockN{Tianchen Ji, Roy Dong, and Katherine Driggs-Campbell}
\authorblockA{Department of Electrical and Computer Engineering, University of Illinois Urbana-Champaign}}


%

\maketitle

\begin{abstract}
The number of multi-robot systems deployed in field applications has increased dramatically over the years. Despite the recent advancement of navigation algorithms, autonomous robots often encounter challenging situations where the control policy fails and the human assistance is required to resume robot tasks. Human-robot collaboration can help achieve high-levels of autonomy, but monitoring and managing multiple robots at once by a single human supervisor remains a challenging problem. Our goal is to help a supervisor decide which robots to assist in which order such that the team performance can be maximized. We formulate the one-to-many supervision problem in uncertain environments as a dynamic graph traversal problem. An approximation algorithm based on the profitable tour problem on a static graph is developed to solve the original problem, and the approximation error is bounded and analyzed. Our case study on a simulated autonomous farm demonstrates superior team performance than baseline methods in task completion time and human working time, and that our method can be deployed in real-time for robot fleets with moderate size.
\end{abstract}

\IEEEpeerreviewmaketitle

\section{Introduction}
Multi-robot systems have been applied to different tasks, such as search and rescue~\cite{drew2021multi,liu2013robotic}, delivery~\cite{camisa2021multi,mathew2015optimal}, mechanical weeding~\cite{mcallister2019agbots}, and underwater missions~\cite{das2016cooperative,kulkarni2010task}. Although recent research efforts have made noteworthy progress on developing trustworthy robot autonomy~\cite{kayacan2018embedded,sivakumar2021learned,zhang2020high}, the deployment of low-cost robots in real-world environments has shown that they usually face difficulties in completing the tasks independently~\cite{rosenfeld2017intelligent}. For example, due to the environmental complexity and terrain variability in agricultural environments, compact field robots (Figure~\ref{subfig:terrasentia}) deployed between rows of crops may fail the navigation task and get stuck in error states (Figure~\ref{subfig:robot-failure}), in which physical assistance from a human supervisor is required to continue the robot task~\cite{ji2020multi,ji2022proactive}. To ensure the smooth operation of such multi-robot systems, the supervision and management of robot fleets are necessary in the presence of imperfect autonomy. However, optimally identifying which robot to assist in which order in an uncertain environment is a challenging task for human supervisors. As shown in recent user studies on operating multi-robot teams, a single operator can get overwhelmed by the number of requests and messages, resulting in sub-optimal performance~\cite{chen2011effects,chien2013imperfect,lewis2013human}. Moreover, the increase of the number of robots can cause further performance degradation of the human-robot team~\cite{olsen2004fan}.

As opposed to purely relying on the supervisor to make decisions, algorithmic approaches have been proposed to provide advice on which robot to assist and when~\cite{drew2021multi}. Many previous works modeled the supervision problem of multi-robot systems as a Markov Decision Process (MDP)~\cite{dahiya2021scalable,rosenfeld2017intelligent}. However, calculating the optimal solution to such MDPs is intractable due to the exponential size of the state space, the high uncertainty induced by the environment, and the complexity of multi-robot problems~\cite{korsah2013comprehensive,kumar2012survey}. Instead, myopic/greedy approaches were used under the assumption that the robots can be managed through \textit{teleoperation}. One benefit that \textit{teleoperation} brings is that which robot the supervisor chooses to assist at the current time step will not affect the cost of assisting other robots at the next time step. However, in robotic applications where \textit{physical} assistance is required, approaching and rescuing a robot can change the cost of rescuing other robots as the physical distance between the supervisor and each robot varies with the action of the supervisor. As a result, the local optimum of \textit{a robot}, provided by myopic/greedy approaches, can be significantly different from the global optimum of a \textit{robot fleet}.

An alternative solution is to learn rational human decisions on supervision problem in easy settings with a few robots and then deploy the learned model in hard settings to model the user's choice of which robot to assist~\cite{swamy2020scaled,xu2016towards}. The underlying assumption behind this approach is that the human decision is approximately optimal in the sense of maximizing the performance of the human-robot team. However, such human decision modeling also benefits from the \textit{teleoperation} of robot fleets: the supervisor can easily decide which robot to assist solely based on how useful it would be to assist each robot without the need to consider the cost of assistance, which is made constant for each robot by \textit{teleoperation}. With varying cost in robotic systems where \textit{physical} assistance is required, the supervisor can struggle with making optimal decisions that can benefit the robot fleet as a whole.

\begin{figure}[t]
  \centering
  \begin{subfigure}[b]{0.49\linewidth}
    \includegraphics[width=\linewidth]{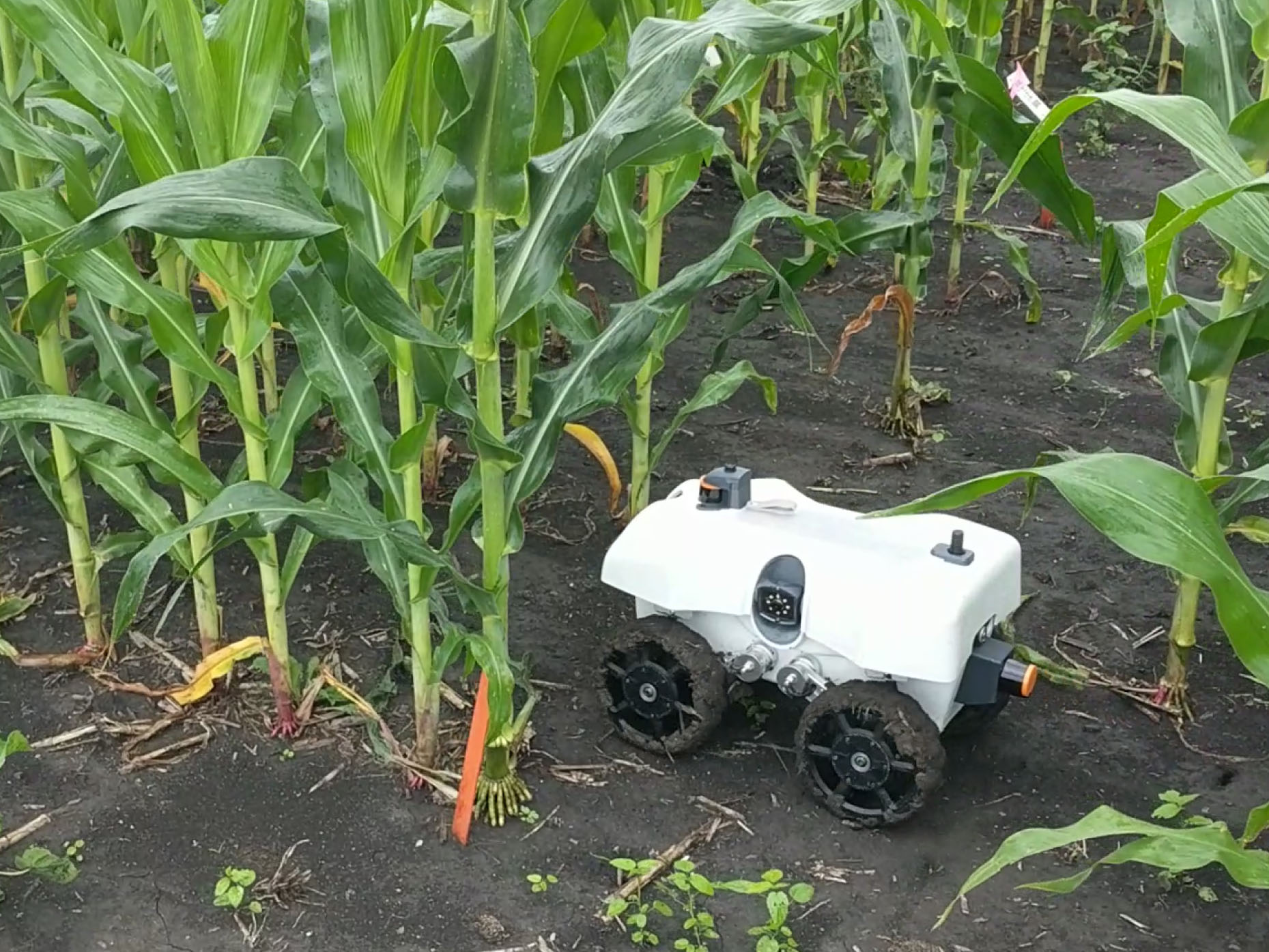}
    \caption{TerraSentia robot.}
    \label{subfig:terrasentia}
  \end{subfigure}
  \begin{subfigure}[b]{0.49\linewidth}
    \includegraphics[width=\linewidth]{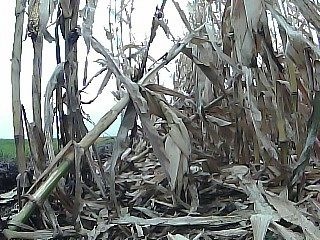}
    \caption{Untraversable obstacle.}
    \label{subfig:robot-failure}
  \end{subfigure}
  \caption{\textbf{(a)} The field robot, TerraSentia~\cite{kayacan2018embedded}, navigates between rows of crops for plant phenotyping. \textbf{(b)} Robot navigation failure may occur due to untraversable obstacles, such as fallen plants and dense weeds.}
  \label{fig:field-robot-example}
  \vspace{-3mm}
\end{figure}

A typical example in which \textit{physical} assistance is required involves robot navigation in uncertain environments (e.g., agricultural fields). \textit{Teleoperation} is inadequate and/or impossible for robot assistance for two main reasons. First, a common failure mode of such field robots is encountering untraversable obstacles blocking the narrow path (Figure~\ref{subfig:robot-failure}), the removal of which is beyond the robot autonomy and requires external assistance~\cite{ji2020multi}. Second, successful teleoperation often requires rich and detailed feedback signals (e.g., videos and images) from the robots, which is not supported in low-cost multi-robot systems or requires the relocation of robots to a given rendezvous point due to the limited bandwidth of the long range communication network in outdoor environments~\cite{waharte2010supporting,lam2017lora}.

In this paper, we formulate the one-to-many supervision problem of a multi-robot system, to our knowledge for the first time, as a dynamic graph traversal problem (GTP). The supervisor is monitoring the fleet to provide \textit{physical} assistance to each robot when necessary. During operation, a robot may fail the navigation task at each location with different probabilities and can only resume the autonomy with assistance from the supervisor. Figure~\ref{fig:system-overview} presents an overview of the problem setup instantiated on an autonomous farm, showing $n$ robots navigating in the field, moving from start to goal locations respectively. To help the supervisor decide which robot to rescue, we construct a static graph with well-defined node rewards and edge costs from the current state of the human-robot team, solve the optimal path based on an objective maximizing the team performance, and execute the optimal actions for the supervisor in a receding horizon manner. Specifically, our contributions can be summarized as:
\begin{enumerate}
\item
We formulate the supervision of a multi-robot system in uncertain environments as a dynamic graph traversal problem, which is necessary in real-world applications where \textit{teleoperation} is disabled and \textit{physical} assistance is required.
\item
We approximate the solution to the intractable \textit{dynamic} problem by solving a \textit{static} graph traversal problem and develop a bound on the approximation error.
\item
A practical implementation of the static GTP is presented based on the profitable tour problem (PTP).
\item
A simulation environment of an autonomous farm is developed and our case study shows that the proposed method outperforms baseline methods in task completion time and human working time during the human-robot collaboration tasks.
\end{enumerate}

\begin{figure}[t]
  \centering
  \includegraphics[width=0.8\linewidth]{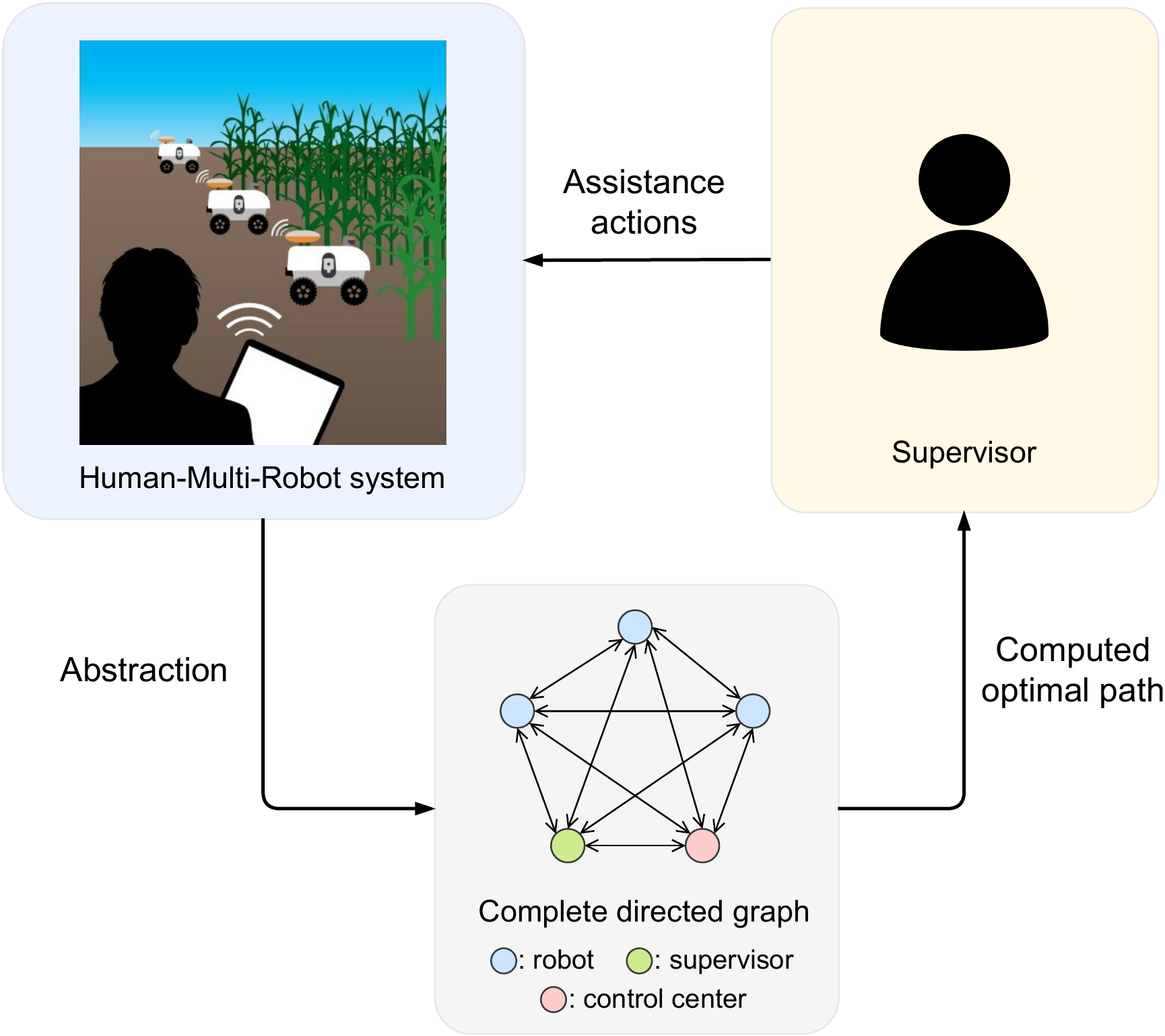}
  \caption{\textbf{Overview of the multi-robot assistance problem for field robots navigating on an autonomous farm.} A complete directed graph is constructed from the current state of the human-robot team, and the optimal path is computed and provided to the supervisor as the actions for robot assistance.}
  \label{fig:system-overview}
  \vspace{-3mm}
\end{figure} 
\section{Related Work}
The problem of multi-robot assistance bares similarities with the disciplines of human-multi-robot team collaboration and task allocation. In this section, we briefly review the related research and introduce the profitable tour problem.

Multi-robot supervision has been widely explored to enable human operators to assist multiple robots such as a team of navigating robots and a team performing search and rescue operations. \citet{rosenfeld2017intelligent} formulate the one-to-many supervision problem as a Myopic Advice Optimization Problem (MYAO), which models the maximization of the operator's performance by selecting when and which advice on robot assistance provided to the operator in a greedy fasion. \citet{dahiya2021scalable} study the problem where each robot is required to complete a sequence of tasks, each characterized by a different probability of successful completion. An index-based policy is then developed to provide advice on which robot requires attention. \citet{swamy2020scaled} learns the decision-making model of human operator with small robot fleets and tries to replicate the behavior for large robot fleets, optimizing based on the operator's internal utility function. However, the problem formulation of these approaches depends on the availability of teleoperation of robot fleets and does not encompass situations where \textit{physical} assistance is required.

The task allocation problem forms another class of problems similar to the supervision problem, which can alternatively be viewed as allocating the supervisor to a series of robots in order. ~\citet{lee2018resource} considers the problem of allocating multiple robots to multiple tasks (e.g., delivery missions) and proposes a resource-based task allocation algorithm to improve the overall task completion time. ~\citet{liu2018reinforcement} proposes a reinforcement-learning-based resource allocation scheme to efficiently allocate the resources of a cloud computing service platform to different requests. However, our problem setup is different because the cost of allocation/supervision varies over time with the movement of each robot, and our problem presents a collaborative task, where both the robots and the supervisor are working to achieve a common goal.

The closest prior work that involves stochastic and non-stationary rewards and costs in task allocation for multi-robot systems is Gittins-Index-based approach used in mechanical robotic weeding, where a team of autonomous robots are allocated to different rows to weed agricultural fields~\cite{mcallister2019agbots}. The reward of weeding a row varies over time according to weed growth models, and the cost of each allocation is dependent on the distance from the robot to the candidate row. To allocate weeding tasks to each robot, an optimization-based approach over a \textit{one-step} planning horizon using Gittins Index is utilized. In this work, we perform \textit{long-horizon} planning by constructing a graph from the human-robot team and provide advice to the supervisor in a receding horizon manner.

The profitable tour problem is a variant of the well-known traveling salesman problem (TSP), which is also referred to as prize collecting TSP~\cite{balas1989prize}. The objective is to find a tour that visits a subset of vertices such that the sum of rewards associated with vertices in the tour minus the length of the tour is as large as possible. In addition to exact solutions, approximation algorithms have been proposed to solve the PTP under the assumption that the edge costs satisfy the triangle inequality~\cite{archer2011improved,bienstock1993note,goemans1995general,nguyen2010primal}. However, the formulation of the PTP considers time-invariant node rewards and edge costs, and it remains unclear how well the solution on such static graphs generalizes to more practical dynamic graphs.
\section{Problem Formulation}
\label{sec:problem-formulation}
Our goal is to help the supervisor choose which robot to rescue at any given time. To this end, we formulate the one-to-many supervision over a robot fleet as a dynamic graph traversal problem.

We model the human-robot team as a directed graph $G=(V,A)$, where $V = \{0,\dots,n + 1\}$ is the set of vertices and $A$ is the set of arcs. Vertices in $N = V \backslash \{0, n+1\} = \{1,\dots,n\}$ correspond to the robots, vertex $0$ corresponds to the supervisor, and vertex $n+1$ corresponds to the control center, in which the supervisor rests when not assisting any robots. As a result, a path to rescue the robots should start from the supervisor and end with the control center. Formally, we define a \textit{path} $\pi \coloneqq (v_k)_{k=0}^{L}$ in the graph as a sequence of \textit{vertices} $v_k \in V$, where $L$ is the length of the path. Henceforth, we use $v_k^\pi$ and $a_k^\pi \coloneqq (v_k^\pi, v_{k+1}^\pi)$ to denote the $k$-th vertex and the $k$-th arc on the path $\pi$, respectively. In our planning problem, a path $\pi$ is a \textit{valid} path if $v_0^\pi = 0$ and $v_L^\pi = n + 1$.

A \textit{time-varying} nonnegative scalar reward $r_i(t) \in \mathbb{R}$ is associated with each robot $i \in N$, while a reward $r_0(t)=0, \, \forall t \ge 0$ and $r_{n+1}(t)=0, \, \forall t \ge 0$ is associated with the supervisor and the control center, respectively. A \textit{time-varying} nonnegative scalar cost $c_{ij}(t) \in \mathbb{R}$ is associated with each arc $(i,j) \in A$. Here, we define the reward $r_i(t)$ for each robot as the \textit{net} expected distance that the robot can travel with supervisor intervention, i.e. the reward is the expected distance that the robot can travel if visited, minus the distance if ignored by the supervisor. The definition of the reward reflects the ``value" of visiting a robot at a given time. Thus, the reward $r_i(t), i \in N$ is positive only when the robot fails the navigation task and is kept zero during normal operation. We define the cost $c_{ij}(t)$ as the time required to travel between the two nodes $i, j \in V$ in the graph\footnote{Note that the graph $G$ is defined as a directed graph, in which the edge cost can be generalized to be asymmetric, i.e., $c_{ij}(t) \neq c_{ji}(t), \, \forall (i, j) \in A$.}. We note that our problem formulation differs from the classical PTP and TSP in that the rewards and costs are dynamic. For notational brevity, we define $t_k^\pi$ to represent the time at which the supervisor reaches the $k$-th vertex $v_k^\pi$ on the path $\pi$:
\begin{equation}
t_k^\pi = \sum_{i=0}^{k-1} c_{a_i^\pi}(t_i^\pi), \quad t_0^\pi = 0, \quad k = 0, 1, \dots, L.
\end{equation}

To plan a path to rescue the robots with minimum time cost, we define the value of a path $\pi$ by:
\begin{equation}
\label{eq:value-function-tv-complete}
Q^\pi = \sum_{k=0}^{L} e^{-\lambda t_k^\pi} r_{v_k^\pi} (t_k^\pi) - \sum_{k=0}^{L-1} e^{-\lambda t_k^\pi} \mu c_{a_k^\pi} (t_k^\pi),
\end{equation}
where $\mu$ is the weighting factor and $\lambda$ is the discount factor. The first term in the value function~(\ref{eq:value-function-tv-complete}) collects the total discounted reward by visiting the robots on the path, while the second term represents the total discounted traveling cost of the supervisor. The hyperparameter $\mu$ controls the relative weight between the value of rescuing the robot and the human labor cost. A higher unit profit from engaging the robot autonomy and a lower human labor cost can result in a lower $\mu$, and vice versa. For example, if the human labor cost is negligible compared to the profit made by a normal working robot, the edge cost incurred by adding a robot to the path can be ignored and thus the optimal path must contain all the robots that have positive rewards. Henceforth, we use $\hat{r}_i(t)$ and $\hat{c}_{ij}(t)$ to denote $r_i(t)$ and $\mu c_{ij}(t)$, respectively, for the brevity of notation. The value function~(\ref{eq:value-function-tv-complete}) can be rewritten in the form:
\begin{equation}
\label{eq:value-function-tv}
Q^\pi = \sum_{k=0}^{|\pi|} e^{-\lambda t_k^\pi} \hat{r}_{v_k^\pi} (t_k^\pi) - \sum_{k=0}^{|\pi|-1} e^{-\lambda t_k^\pi} \hat{c}_{a_k^\pi} (t_k^\pi),
\end{equation}
where $|\pi|$ denotes the length of the path. We assume that the reward of each robot can be collected at most once. A more general case where each robot can be visited repeatedly over time will be discussed in Section~\ref{sec:method}.

The optimization problem of interest is to find a \textit{valid} path for the supervisor that maximizes the value function~(\ref{eq:value-function-tv}):
\begin{equation}
\label{prob:dynamic-gtp}
\begin{aligned}
\max_{\pi} \quad &Q^\pi \\
\text{s.t.} \quad &v_i^\pi \neq v_j^\pi, \quad \forall \, i,j \in \{0,\dots,|\pi|\}, \, i \neq j \\
&v_0^\pi = 0, \; v_{|\pi|}^\pi = n+1.
\end{aligned}
\end{equation}
Let $\pi^*$ be the optimal path in the optimization problem~(\ref{prob:dynamic-gtp}).

To our knowledge, our work is the first to consider a graph traversal problem with time-varying rewards and costs, and no exact approaches to the \textit{dynamic} GTP~(\ref{prob:dynamic-gtp}) exists other than the enumeration over the space of all paths. In the following section, we approximate the intractable \textit{dynamic} GTP~(\ref{prob:dynamic-gtp}) to a \textit{static} GTP, whose solution can be obtained efficiently using mixed-integer programming~\cite{gurobi}.
\section{Methodology}
\label{sec:method}
In this section, we present an approximation algorithm for the dynamic GTP~(\ref{prob:dynamic-gtp}) based on the solution to the profitable tour problem. The bound on the approximation error is developed and the practical implementation will be detailed.

\subsection{Static GTP Approximation}
We now introduce a \textit{static} graph traversal problem, whose node rewards and edge costs are time invariant. We consider the directed graph $G$ as defined in Section~\ref{sec:problem-formulation}. A \textit{constant} nonnegative reward $r_i \in \mathbb{R}$ is associated with each robot $i \in N$, while a reward $r_0=0$ and $r_{n+1}=0$ is associated with the supervisor and the control center, respectively. A \textit{constant} nonnegative traveling cost $c_{ij} \in \mathbb{R}$ is associated with each arc $(i,j) \in A$. The definition of the node reward and the edge cost are identical to those in Section~\ref{sec:problem-formulation}.

For the \textit{static} graph traversal problem, we define the value of a path $\pi$ as:
\begin{equation}
\label{eq:value-function-static}
Q_\text{s}^\pi = \sum_{k=0}^{|\pi|} \hat{r}_{v_k^\pi} - \sum_{k=0}^{|\pi|-1} \hat{c}_{a_k^\pi},
\end{equation}
where $\hat{r}_i$ and $\hat{c}_{ij}$ are the node rewards and scaled edge costs as defined in equation~(\ref{eq:value-function-tv}). Similar to the value function in the dynamic GTP, the objective is to maximize the reward collected and minimize the path cost. Under the same assumption that the reward of each robot can be collected at most once, the optimization problem for the \textit{static} GTP can be formulated as:
\begin{equation}
\label{prob:static-gtp}
\begin{aligned}
\max_{\pi} \quad &Q_\text{s}^\pi \\
\text{s.t.} \quad &v_i^\pi \neq v_j^\pi, \quad \forall \, i,j \in \{0,\dots,|\pi|\}, \, i \neq j \\
&v_0^\pi = 0, \; v_{|\pi|}^\pi = n+1.
\end{aligned}
\end{equation}
Let $\pi_\text{s}^*$ be the optimal path in the optimization problem~(\ref{prob:static-gtp}).

To overcome the intractability of the original dynamic problem, we approximate the solution to problem~(\ref{prob:dynamic-gtp}) by solving the static GTP~(\ref{prob:static-gtp}) based on the following assumptions:
\begin{assumption}
\label{asm:1}
In the dynamic GTP, the change rate of the node reward and edge cost are bounded above by $|\hat{r}_i(t) - \hat{r}_i(0)| \le \alpha t, \, |\hat{c}_{ij}(t) - \hat{c}_{ij}(0)| \le \beta t, \, \forall \, i,j \in V, \, \forall \, t \ge 0$, where $\alpha$ and $\beta$ are nonnegative real numbers. In the static GTP, the difference between the node reward and the edge cost in the graph is bounded above by $\max_{i,j} |\hat{r}_i - \hat{c}_{ij}| \le \epsilon, \, \forall \, i,j \in V$.
\end{assumption}

\begin{assumption}
\label{asm:2}
The traveling time $\Delta t$ between any two nodes in the graph $G$ is bounded above by $\Delta t \le \overline{\Delta t}$.
\end{assumption}

In practice, $\alpha$ and $\beta$ in Assumption~\ref{asm:1} bound the rates at which the reward and cost can change due to the movement of the robot over time, and Assumption~\ref{asm:2} suggests that the field on which the robots can navigate has a limited size. We construct the \textit{static} graph as the initial condition of the \textit{dynamic} graph, i.e., $\hat{r}_i = \hat{r}_i(0), \, \hat{c}_{ij} = \hat{c}_{ij}(0), \, \forall \, i, j \in V$. We now bound the approximation error of the optimal solution of the static graph traversal problem~(\ref{prob:static-gtp}):
\begin{theorem}
\label{theorem:1}
Consider the dynamic graph traversal problem~(\ref{prob:dynamic-gtp}) and the static graph traversal problem~(\ref{prob:static-gtp}). Let $Q^{\pi^*}$ and $Q_\text{s}^{\pi_\text{s}^*}$ be the maximum value of $Q^{\pi}$ and $Q_\text{s}^{\pi}$ over all paths, respectively. The solution to the static problem approximates the solution to the dynamic problem with the bound:
\begin{equation}
\label{ineq:theorem}
|Q_\text{s}^{\pi_\text{s}^*} - Q^{\pi^*}|
\le
(\alpha + \beta) \frac{n+1}{\lambda e} + \sum_{k=0}^{n} (1 - e^{-\lambda k \overline{\Delta t}}) \epsilon,
\end{equation}
where $n$ is the number of robots.
\end{theorem}
\begin{proof}
For any valid path $\pi$, the value function~(\ref{eq:value-function-tv}) and~(\ref{eq:value-function-static}) can be written as:
\begin{align*}
Q^{\pi}
&=
\sum_{k=0}^{|\pi|-1} e^{-\lambda t_k^{\pi}} \left( \hat{r}_{v_k^{\pi}} (t_k^{\pi}) - \hat{c}_{a_k^{\pi}} (t_k^{\pi}) \right), \\
Q_\text{s}^\pi
&=
\sum_{k=0}^{|\pi|-1} \left( \hat{r}_{v_k^\pi} - \hat{c}_{a_k^\pi} \right),
\end{align*}
by the fact that the terminal node on a valid path must be the control center, whose reward is always zero.

By applying $\pi_\text{s}^*$, which is the optimal solution to~(\ref{prob:static-gtp}), to both the static and dynamic problem, we have:
\begin{equation}
\label{eq:proof-1}
\begin{aligned}
|Q_\text{s}^{\pi_\text{s}^*} - Q^{\pi_\text{s}^*}|
&=
\left| \sum_{k=0}^{|\pi_\text{s}^*|-1} \left( \hat{r}_{v_k^{\pi_\text{s}^*}} - e^{-\lambda t_k^{\pi_\text{s}^*}} \hat{r}_{v_k^{\pi_\text{s}^*}}(t_k^{\pi_\text{s}^*}) \right) \right. \\
&\qquad \quad \, \left. - \left( \hat{c}_{a_k^{\pi_\text{s}^*}} - e^{-\lambda t_k^{\pi_\text{s}^*}} \hat{c}_{a_k^{\pi_\text{s}^*}}(t_k^{\pi_\text{s}^*}) \right) \right|.
\end{aligned}
\end{equation}
By Assumption~\ref{asm:1}, we have
$
\left| \hat{r}_{v_k^{\pi_\text{s}^*}}(t_k^{\pi_\text{s}^*}) - \hat{r}_{v_k^{\pi_\text{s}^*}} \right| \le \alpha t_k^{\pi_\text{s}^*}
$, which is equivalent to:
\begin{equation*}
\left| \hat{r}_{v_k^{\pi_\text{s}^*}} - e^{-\lambda t_k^{\pi_\text{s}^*}} \hat{r}_{v_k^{\pi_\text{s}^*}}(t_k^{\pi_\text{s}^*}) - \hat{r}_{v_k^{\pi_\text{s}^*}} ( 1 - e^{-\lambda t_k^{\pi_\text{s}^*}} ) \right|
\le 
e^{-\lambda t_k^{\pi_\text{s}^*}} \alpha t_k^{\pi_\text{s}^*},
\end{equation*}
indicating that the first term on the R.H.S. of the equation~(\ref{eq:proof-1}) is centered around $\hat{r}_{v_k^{\pi_\text{s}^*}} ( 1 - e^{-\lambda t_k^{\pi_\text{s}^*}} )$ within a range of $e^{-\lambda t_k^{\pi_\text{s}^*}} \alpha t_k^{\pi_\text{s}^*}$. Similarly, for the edge cost, we have:
\begin{equation*}
\left| \hat{c}_{a_k^{\pi_\text{s}^*}} - e^{-\lambda t_k^{\pi_\text{s}^*}} \hat{c}_{a_k^{\pi_\text{s}^*}}(t_k^{\pi_\text{s}^*}) - \hat{c}_{a_k^{\pi_\text{s}^*}} ( 1 - e^{-\lambda t_k^{\pi_\text{s}^*}} ) \right|
\le 
e^{-\lambda t_k^{\pi_\text{s}^*}} \beta t_k^{\pi_\text{s}^*},
\end{equation*}
indicating that the second term on the R.H.S of the equation~(\ref{eq:proof-1}) is centered around $\hat{c}_{a_k^{\pi_\text{s}^*}} ( 1 - e^{-\lambda t_k^{\pi_\text{s}^*}} )$ within a range of $e^{-\lambda t_k^{\pi_\text{s}^*}} \beta t_k^{\pi_\text{s}^*}$. Therefore, from equation~(\ref{eq:proof-1}) we have:
\begin{equation*}
\begin{aligned}
&|Q_\text{s}^{\pi_\text{s}^*} - Q^{\pi_\text{s}^*}|
\le
\sum_{k=0}^{|\pi_\text{s}^*|-1} \left| \left( \hat{r}_{v_k^{\pi_\text{s}^*}} - e^{-\lambda t_k^{\pi_\text{s}^*}} \hat{r}_{v_k^{\pi_\text{s}^*}}(t_k^{\pi_\text{s}^*}) \right) \right. \\
&\qquad \qquad \qquad \qquad \qquad \qquad \left. - \left( \hat{c}_{a_k^{\pi_\text{s}^*}} - e^{-\lambda t_k^{\pi_\text{s}^*}} \hat{c}_{a_k^{\pi_\text{s}^*}}(t_k^{\pi_\text{s}^*}) \right) \right| \\
&\le
\sum_{k=0}^{|\pi_\text{s}^*|-1} ( 1 - e^{-\lambda t_k^{\pi_\text{s}^*}}) \left| \hat{r}_{v_k^{\pi_\text{s}^*}} - \hat{c}_{a_k^{\pi_\text{s}^*}} \right| + (\alpha + \beta) e^{-\lambda t_k^{\pi_\text{s}^*}} t_k^{\pi_\text{s}^*}
\end{aligned}
\end{equation*}
Consider the function $f(x) = e^{-\lambda x} x, \, x \ge 0$. By taking the derivative to zero, we observe that $f(x)$ reaches a maximum of $1/(\lambda e)$ at $x=1/\lambda$. Therefore, we have:
\begin{equation}
\label{ineq:proof-1}
\begin{aligned}
|Q_\text{s}^{\pi_\text{s}^*} - Q^{\pi_\text{s}^*}|
&\le
(\alpha + \beta) \frac{n + 1}{\lambda e} \\ 
&\qquad \qquad \quad + \sum_{k=0}^n ( 1 - e^{-\lambda t_k^{\pi_\text{s}^*}}) \left| \hat{r}_{v_k^{\pi_\text{s}^*}} - \hat{c}_{a_k^{\pi_\text{s}^*}} \right| \\
&\le
(\alpha + \beta) \frac{n + 1}{\lambda e} + \sum_{k=0}^n ( 1 - e^{-\lambda t_k^{\pi_\text{s}^*}}) \epsilon \\
&\le
(\alpha + \beta) \frac{n + 1}{\lambda e} + \sum_{k=0}^n ( 1 - e^{-\lambda k \overline{\Delta t}}) \epsilon,
\end{aligned}
\end{equation}
where the second inequality follows from Assumption~\ref{asm:1} and the last inequality follows from Assumption~\ref{asm:2} and the fact that $t_k^{\pi_\text{s}^*} \le k \overline{\Delta t}$.

Similarly, by applying $\pi^*$, which is the optimal solution to~(\ref{prob:dynamic-gtp}), to both the static and dynamic problem, we have:
\begin{equation}
\label{ineq:proof-2}
|Q_\text{s}^{\pi^*} - Q^{\pi^*}|
\le
(\alpha + \beta) \frac{n + 1}{\lambda e} + \sum_{k=0}^n ( 1 - e^{-\lambda k \overline{\Delta t}}) \epsilon.
\end{equation}
Note that $\pi^*$ is suboptimal to the static problem~(\ref{prob:static-gtp}) and $\pi_\text{s}^*$ is suboptimal to the dynamic problem~(\ref{prob:dynamic-gtp}):
\begin{equation}
\label{ineq:proof-3}
Q_\text{s}^{\pi^*} \le Q_\text{s}^{\pi_\text{s}^*}, \quad Q^{\pi_\text{s}^*} \le Q^{\pi^*}.
\end{equation}
From the inequality~(\ref{ineq:proof-1}),~(\ref{ineq:proof-2}), and~(\ref{ineq:proof-3}), we conclude that:
\begin{equation*}
|Q_\text{s}^{\pi_\text{s}^*} - Q^{\pi^*}|
\le
(\alpha + \beta) \frac{n+1}{\lambda e} + \sum_{k=0}^{n} (1 - e^{-\lambda k \overline{\Delta t}}) \epsilon.
\end{equation*}
\end{proof}

An interesting fact that we observe on the bound developed in Theorem~\ref{theorem:1} is that the approximation error is separated into two parts and the discount factor $\lambda$ creates a tension between the two terms on the right hand side of the inequality~(\ref{ineq:theorem}). The first term encapsulates the error generated by the time-varying nature of the rewards and costs in the dynamic problem and can be damped by a large $\lambda$, by applying which the further nodes on the path could make less contribution to the value function compared to the nearer nodes. The second term represents the error originated from the fact that $\pi^*$ is optimized for discounted rewards while $\pi_\text{s}^*$ is optimized for undiscounted rewards and can be damped by a small $\lambda$. Consider a dynamic graph with a change rate of zero, i.e., $\alpha=\beta=0$. The optimal solution to the static and dynamic problem can still be different due to the different definitions of the value function. With an additional condition that $\lambda=0$, however, the value function~(\ref{eq:value-function-tv}) and~(\ref{eq:value-function-static}) of the two problems will be identical and the gap between the optimal solution $Q_\text{s}^{\pi_\text{s}^*}$ and $Q^{\pi^*}$ will be zero, which is also verified by the bound in Theorem~\ref{theorem:1}.

We note that in real-world applications, robots may fail repeatedly, which, in our setting, means that robots can be visited multiple times for reward. Formally, the optimal path may contain loops. However, if the robot autonomy is relatively reliable with rare navigation failures, the second visit to a robot will be necessary only after a long time since the first visit, i.e., $t_{k_1}^{\pi^*} \gg t_{k_2}^{\pi^*}, \, \forall \, k_1, k_2 \in \{0, \dots, |\pi^*|\}, \, k_1 \neq k_2, \, v_{k_1}^{\pi^*} = v_{k_2}^{\pi^*}$. As a result, the discounted reward collected after the time at which the loop is formed is negligible according to equation~(\ref{eq:value-function-tv}). Therefore, $Q_\text{s}^{\pi_\text{s}^*}$ can still approximate the optimal value function of the dynamic graph traversal problem~(\ref{prob:dynamic-gtp}) well with the bound given by equation~(\ref{ineq:theorem}).

\subsection{Profitable Tour Problem}
\label{subsec:ptp}
Although the optimization in the static GTP is still performed over the space of all paths, problem~(\ref{prob:static-gtp}) can be formulated as a profitable tour problem.

The PTP is the problem of finding a route that maximizes the difference between the total collected reward and the total traveling cost. In the remaining part of this section, we first introduce some useful notations and then formulate the static GTP~(\ref{prob:static-gtp}) as a modified PTP.

For any subset of vertices $S \subset V$, we define $\delta^+ (S) = \{ (i, j) \in A: i \in S, j \notin S \}$ and $\delta^- (S) = \{ (i, j) \in A: i \notin S, j \in S \}$. For brevity, we will use the notation $\delta^+ (i)$ and $\delta^- (i)$ when $S=\{ i \}$. We define the variable $y_i$ as a binary variable equal to $1$ if vertex $i \in V$ is visited by the path, and $0$ otherwise. Similarly, we define the variable $x_{ij}$ as a binary variable equal to $1$ if arc $(i, j) \in A$ is traversed by the supervisor, and $0$ otherwise.

The mathematical programming formulation of the static GTP~(\ref{prob:static-gtp}) is the following:
\begin{align}
\label{eq:ptp-obj}
\max &\quad \sum_{i \in V} \hat{r}_i y_i - \sum_{(i, j) \in A} \hat{c}_{ij} x_{ij} \\
\label{eq:ptp-robot-leaving-cons}
\text{s.t.} \quad &\sum_{(i, j) \in \delta^+ (i)} x_{ij} = y_i, \quad \forall \, i \in N, \\
\label{eq:ptp-robot-entering-cons}
&\sum_{(j, i) \in \delta^- (i)} x_{ji} = y_i, \quad \forall \, i \in N, \\
\label{eq:ptp-human-cons}
&\sum_{(i, j) \in \delta^+ (i)} x_{ij} = 1, \sum_{(j, i) \in \delta^- (i)} x_{ji} = 0, \; i = 0, \\
\label{eq:ptp-cc-cons}
&\sum_{(i, j) \in \delta^+ (i)} x_{ij} = 0, \sum_{(j, i) \in \delta^- (i)} x_{ji} = 1, \; i = n + 1, \\
\label{eq:ptp-subtour-elimination-cons}
&\sum_{(i, j) \in \delta^+ (S)} x_{ij} \ge y_b, \quad \forall \, S \subseteq N, \; b \in S, \\
\label{eq:ptp-node-var}
&y_i \in \{ 0, 1 \}, \quad \; \, \forall \, i \in V, \\
\label{eq:ptp-edge-var}
&x_{ij} \in \{ 0, 1 \}, \quad \forall \, (i, j) \in A.
\end{align}
The objective function~(\ref{eq:ptp-obj}) maximizes the difference between collected reward and traveling cost. Constraints~(\ref{eq:ptp-robot-leaving-cons}) and~(\ref{eq:ptp-robot-entering-cons}) ensure that one arc leaves and one arc enters each visited robot, respectively. Constraints~(\ref{eq:ptp-human-cons}) and~(\ref{eq:ptp-cc-cons}) ensure that the path starts from the supervisor and ends at the control center, respectively. Subtours are eliminated through~(\ref{eq:ptp-subtour-elimination-cons}). Finally,~(\ref{eq:ptp-node-var}) and~(\ref{eq:ptp-edge-var}) are variable definitions.

The optimization problem~\cref{eq:ptp-obj,eq:ptp-robot-leaving-cons,eq:ptp-robot-entering-cons,eq:ptp-human-cons,eq:ptp-cc-cons,eq:ptp-subtour-elimination-cons,eq:ptp-node-var,eq:ptp-edge-var} is a variant of the profitable tour problem and can be solved efficiently via mixed-integer programming~\cite{archetti2014chapter}. To retrieve the optimal path, we expand the path from vertex $0$, the supervisor, to vertex $n+1$, the control center, by following the arcs $x_{ij}$ that are assigned with a value of $1$.

To plan actions for the supervisor to complete the task, we solve the PTP problem~\cref{eq:ptp-obj,eq:ptp-robot-leaving-cons,eq:ptp-robot-entering-cons,eq:ptp-human-cons,eq:ptp-cc-cons,eq:ptp-subtour-elimination-cons,eq:ptp-node-var,eq:ptp-edge-var} at each time step, command the supervisor to rescue the first robot on the optimal path, and continue to plan and execute at the next time step following the framework of model predictive control.

\subsection{Practical Considerations}
Each robot in the fleet is assigned independent navigation tasks and the completion of the entire task is declared only when all the robots reach their respective destinations. To increase the efficiency of the robot fleet, higher priority should be given to the robots that are further away from their destinations than to those that are closer to complete the navigation tasks~\cite{wang2011spatio}. In other words, the robots traversing less distance from the starting point should be favored to be rescued earlier than those traversing longer distance. To this end, let $d_i, i \in N$ be the distance that the robot $i$ has traversed. We modify the definition of the scaled edge cost $\hat{c}_{ij}$ in the objective function~(\ref{eq:ptp-obj}) as follows:
\begin{equation}
\label{eq:modified-cost}
\hat{c}_{ij}
=
\mu c_{ij} (t) + \gamma \max (0, d_i - d_j), \quad \forall i, j \in N,
\end{equation}
where $\gamma$ is a weighting factor. For any two robots $i$ and $j$ in the fleet, if robot $i$ has traversed further distance than robot $j$ from the starting point, the traveling cost from $i$ to $j$ will have an additional penalty term of $\gamma (d_i - d_j)$ while the traveling cost from $j$ to $i$ will remain unchanged. As a result, the supervisor will be encouraged to visit robot $j$ before robot $i$.

Although the optimization problem~\cref{eq:ptp-obj,eq:ptp-robot-leaving-cons,eq:ptp-robot-entering-cons,eq:ptp-human-cons,eq:ptp-cc-cons,eq:ptp-subtour-elimination-cons,eq:ptp-node-var,eq:ptp-edge-var} is tractable using modern programming solvers, the number of constraints on the subtour elimination~(\ref{eq:ptp-subtour-elimination-cons}) grows exponentially with the number of robots. To ensure the real-time capability of the algorithm, we use lazy constraints in the replacement of the full set of constraints~(\ref{eq:ptp-subtour-elimination-cons}), i.e., the constraints are included only when they are violated by solutions found during the optimization~\cite{gurobi}.
\section{Case Study}
\label{sec:case-study}
The PTP is known to provide an optimal solution to a static GTP with constant node rewards and edge costs~\cite{kool2018attention}. In this section, we aim to evaluate the performance of the modified PTP~\cref{eq:ptp-obj,eq:ptp-robot-leaving-cons,eq:ptp-robot-entering-cons,eq:ptp-human-cons,eq:ptp-cc-cons,eq:ptp-subtour-elimination-cons,eq:ptp-node-var,eq:ptp-edge-var} on a practical dynamic graph.

We instantiate the dynamic GTP formulated in Section~\ref{sec:problem-formulation} as a one-to-many supervision problem of a robot fleet on an autonomous farm, in which compact agricultural robots are tasked with navigating through rows of crops for plant phenotyping~\cite{mueller2017robotanist}. Due to the environmental complexity and terrain variability in practice, robots may fail the navigation task in the field, and a supervisor is responsible for monitoring and rescuing the robots to complete the task.

\subsection{Simulation Environment}
We developed a simulation environment for the autonomous farm (Figure~\ref{fig:simulation}) to allow simulations of fleet management algorithms for one-to-many supervision in uncertain environments. We assume that the crops are arranged in evenly spaced vertical rows and the farm is surrounded by free space in which the supervisor and the robots can move without failures. The farm is divided evenly for each robot to conduct plant phenotyping through boustrophedon navigation.

\begin{figure}[t]
  \centering
  \includegraphics[width=0.85\linewidth]{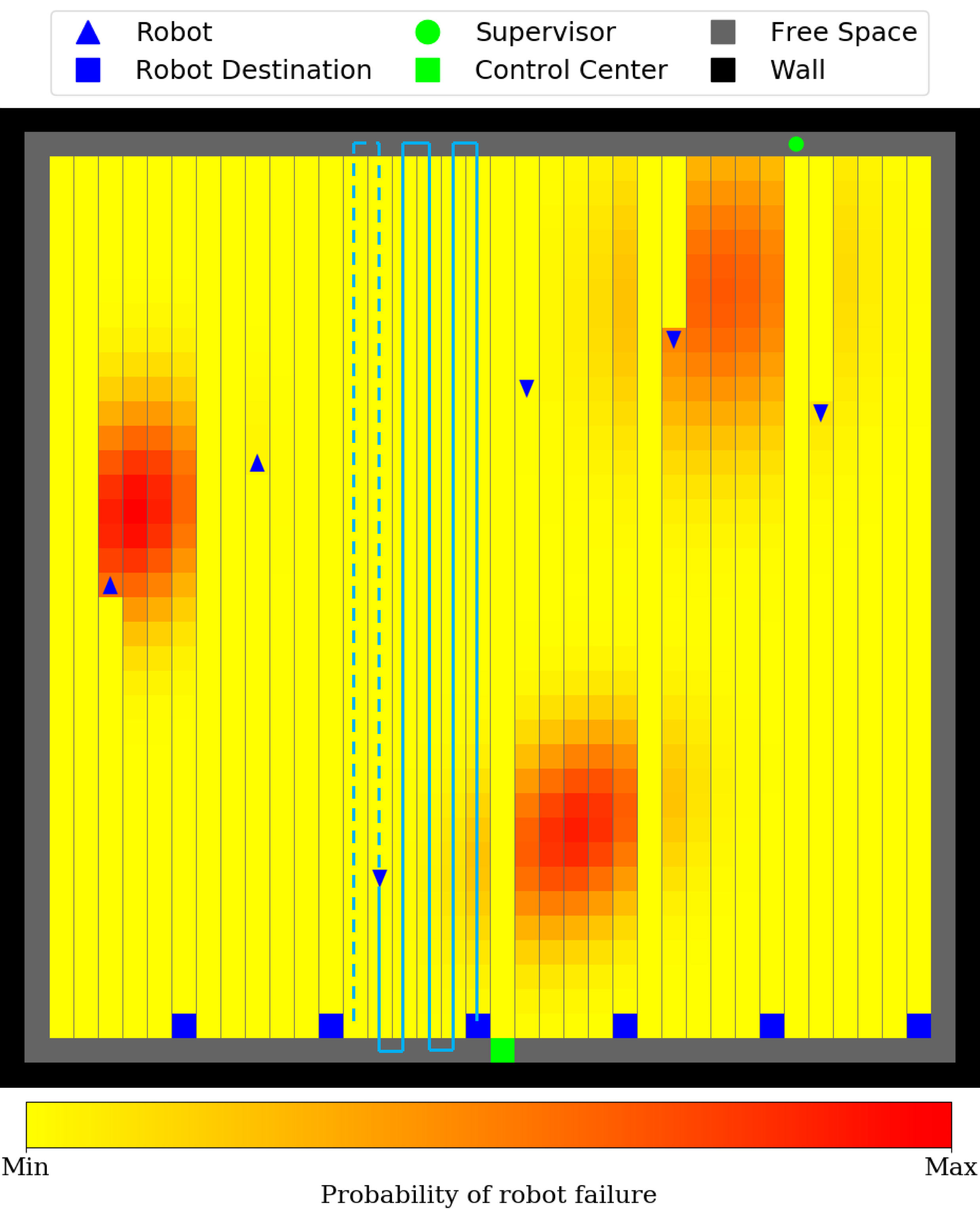}
  \caption{\textbf{Simulation environment of an autonomous farm with six robots.} The vertical grey lines mark the places where the crop is planted and thus represent the boundaries of traversable paths for the robots. The dashed cyan line is the history path traversed by the robot from the starting point, while the solid cyan line is the planned path towards the destination. The history and planned path are visualized only for one robot for clarity.}
  \label{fig:simulation}
\end{figure}

\begin{figure}[t]
  \centering
  \begin{subfigure}[b]{0.19\linewidth}
    \includegraphics[width=\linewidth]{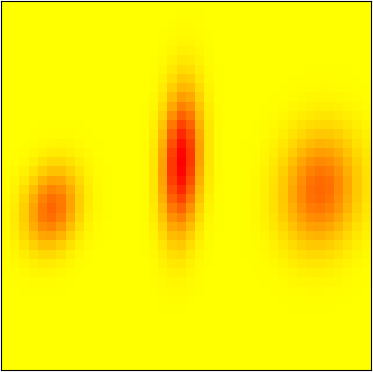}
    \caption{}
  \end{subfigure}
  \begin{subfigure}[b]{0.19\linewidth}
    \includegraphics[width=\linewidth]{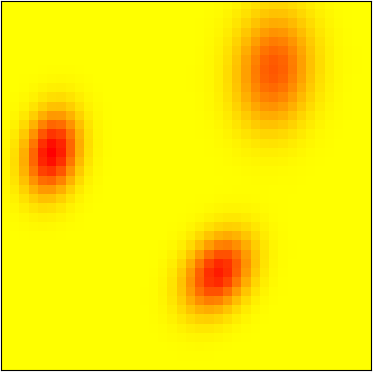}
    \caption{}
  \end{subfigure}
  \begin{subfigure}[b]{0.19\linewidth}
    \includegraphics[width=\linewidth]{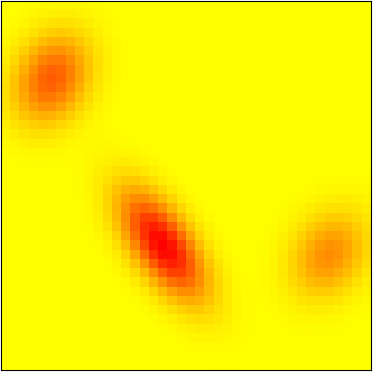}
    \caption{}
  \end{subfigure}
  \begin{subfigure}[b]{0.19\linewidth}
    \includegraphics[width=\linewidth]{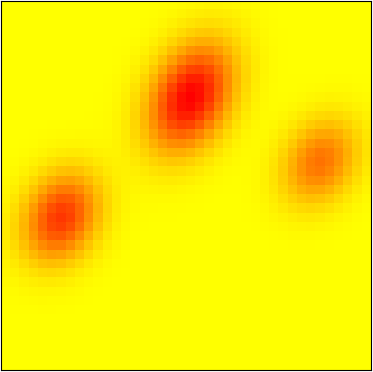}
    \caption{}
  \end{subfigure}
  \begin{subfigure}[b]{0.19\linewidth}
    \includegraphics[width=\linewidth]{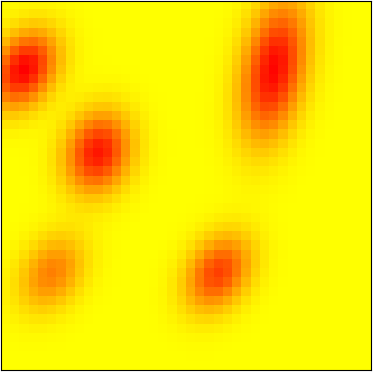}
    \caption{}
  \end{subfigure}
  \caption{\textbf{Agricultural fields used in the case study.} The red cells correspond to a higher probability of robot failure compared to the yellow ones. All the other elements in the simulation environment are omitted.}
  \label{fig:fields}
  \vspace{-3mm}
\end{figure}

In practice, the probability of robot failure is correlated with the navigation difficulty, which may vary over the field due to the presence of weeds, fallen plants, and uneven terrains, to name a few. To incorporate such environmental effects on robot navigation into the simulation, we assign each cell in the farm a nonnegative scalar, which represents the probability of robot failure, based on a bivariate mixed normal distribution with a proper range. Each scalar is further bounded above and below by adjustable thresholds to control the navigation difficulty or the levels of robot autonomy. In Figure~\ref{fig:simulation}, the color of each cell reflects the magnitude of the associated probability of robot failure, i.e., a robot is more likely to fail the navigation task in a red cell than in a yellow cell\footnote{For visualization purposes, the cells that have been traversed by the robots are reassigned with the value of the lower threshold, thus appearing to be yellow in color.}. Upon entering a cell, the robot samples uniformly from the range $[0, 1]$ and takes an action based on the result: proceeds to the next cell if the sampled number is greater than the scalar associated with the current cell or raises a navigation failure otherwise.

The simulation is initialized with the robots at the starting points and the supervisor at the control center. Upon the occurrence of the navigation failure of a robot, the supervisor can approach and rescue the robot to reactivate the autonomy. The robots and the supervisor can both move one cell at a time and can only move along rows of crops, i.e., move vertically, except in the free space. The completion of the task is declared only when all the robots reach the destinations and the supervisor returns to the control center. The expected distance that each robot can travel is computed based on the probability of failure associated with each cell on the planned path, which can be estimated from weed density and/or terrain traversability using aerial systems in practice~\cite{hudjakov2009aerial,sapkota2020mapping}. To plan actions for the supervisor, a directed graph is constructed as described in Section~\ref{sec:problem-formulation} and a receding horizon control strategy is employed as described in Section~\ref{subsec:ptp}.

\subsection{Experimental Setup}
We consider five different fields in the case study (Figure~\ref{fig:fields}), which follow the typical weed density patterns used for weeding robots~\cite{mcallister2020agbots}.

\begin{table}[t]
  \begin{center}
    \caption{Levels of autonomy parameters}
    \label{table:levels-of-autonomy-param}
    \begin{tabular}{ l | c  c  c }
      \toprule
      Levels of autonomy & Low & Mid & High \\
      \midrule
      \rule{-2.5pt}{2ex} Min probability of failure & $0.01$ & $0.01$ & $0$ \\
      Max probability of failure & $0.20$ & $0.15$ & $0.15$ \\
      \bottomrule
    \end{tabular}
  \end{center}
\end{table}

\begin{table}[t]
  \begin{center}
    \caption{Robot fleet size parameters}
    \label{table:robot-fleet-size-param}
    \begin{tabular}{ l | c  c  c }
      \toprule
      Robot fleet size & Small & Mid & Large \\
      \midrule
      \rule{-2.5pt}{2ex} Number of robots & $4$ & $6$ & $9$ \\
      \bottomrule
    \end{tabular}
  \end{center}
  \vspace{-3mm}
\end{table}

We evaluate the performance of the proposed method on the simulated autonomous farm, along with several baseline methods. Let $N_\text{fail} (t) \subseteq N$ be the set of robots that fail the navigation tasks and $v(t) \in V$ be the target of the supervisor given by an algorithm at time $t$. The brief introduction of the baselines is given as follows.
\begin{itemize}
\item
\textit{Greedy-HR (Highest Reward)} assigns the supervisor to the failed robot that generates the highest \textit{one-step} reward:
\begin{equation*}
v(t)
=
\operatorname*{argmax}_j \, \hat{r}_j(t) - \hat{c}_{ij}(t), \; \; i=0, \, j \in N_\text{fail} (t).
\end{equation*}

\item
\textit{Greedy-FTG (Furthest-to-go)} assigns the supervisor to the failed robot that is furthest from the destination:
\begin{equation*}
v(t)
=
\operatorname*{argmin}_j \, d_j, \; \; j \in N_\text{fail} (t).
\end{equation*}

\item
\textit{Greedy-CR} (Closest Robot) assigns the supervisor to the failed robot that is closest to the supervisor:
\begin{equation*}
v(t)
=
\operatorname*{argmin}_j \, c_{ij} (t), \; \; i=0, \, j \in N_\text{fail} (t).
\end{equation*}

\item
\textit{Gittins Index} assigns the supervisor to the failed robot according to the Gittins Index~\cite{weber1992gittins}, which has been applied in weeding algorithms with dynamic rewards~\cite{mcallister2019agbots}:
\begin{equation*}
v(t)
=
\operatorname*{argmax}_j \, \frac{\gamma^{c_{ij}(t)} \hat{r}_j (t)}{\sum_{k=0}^{c_{ij}(t)} \gamma^k}, \; \; i=0, \, j \in N_\text{fail} (t).
\end{equation*}
\end{itemize}
In cases where all the robots are operating normally ($N_\text{fail} (t) = \emptyset$), the above methods assign the supervisor to the control center, i.e., $v(t) = n+1$. To our knowledge, our work is the first to formulate the one-to-many supervision problem as a dynamic GTP, and the above baselines are approaches typically applied in practice to tackle the problem in the absence of a rigorous framework.

Quantitatively, we compare different methods using the following two metrics:
\begin{itemize}
\item
\textit{Task Completion Time}: The time steps elapsed from the start of the simulation until the completion of the task.

\item
\textit{Human Working Time}: The total time steps when the supervisor is not at the control center.
\end{itemize}

The levels of autonomy can affect how much attention and assistance a human supervisor has to provide the robot, while the fleet size has an impact on the difficulty of decision making for multi-robot assistance. To examine the performance of different methods under different conditions, we conduct experiments with varying levels of autonomy and robot fleet size, as described in Table~\ref{table:levels-of-autonomy-param} and Table~\ref{table:robot-fleet-size-param}, respectively. For each combination of field, levels of autonomy, and robot fleet size, each method is evaluated over $10$ trials with different realizations of robot navigation.

\begin{table}[t]
\captionsetup{font=footnotesize}
  \begin{center}
    \caption{Results of fleet management with varying levels of autonomy}
    \label{table:results-mid-fleet-size}
    \begin{tabular}{ l | c  c }
      \toprule
      Methods & Task Completion Time & Human Working Time \\
      \midrule
      \rule{-2.5pt}{2ex} Greedy-HR & $1190.3 \, / \, 1036.1 \, / \, 760.4$ & $1167.0 \, / \, 1014.5 \, / \,725.0$ \\
      Greedy-FTG & $1010.4 \, / \, 868.7 \, / \, 666.2$ & $997.7 \, / \, 854.0 \, / \,641.6$ \\
      Greedy-CR & $929.3 \, / \, 825.5 \, / \, 661.8$ & $900.7 \, / \, 786.6 \, / \, 606.6$ \\
      Gittin's Index & $967.0 \, / \, 817.8 \, / \, 659.9$ & $938.3 \, / \, 776.2 \, / \, 607.0$ \\
      PTP & $\mathbf{905.6} \, / \, \mathbf{788.3} \, / \, \mathbf{626.2}$ & $\mathbf{892.9} \, / \, \mathbf{772.6} \, / \, \mathbf{601.2}$ \\
      \bottomrule
    \end{tabular}
  \end{center}
\end{table}

\begin{figure}[t]
  \centering
  \includegraphics[width=0.9\linewidth]{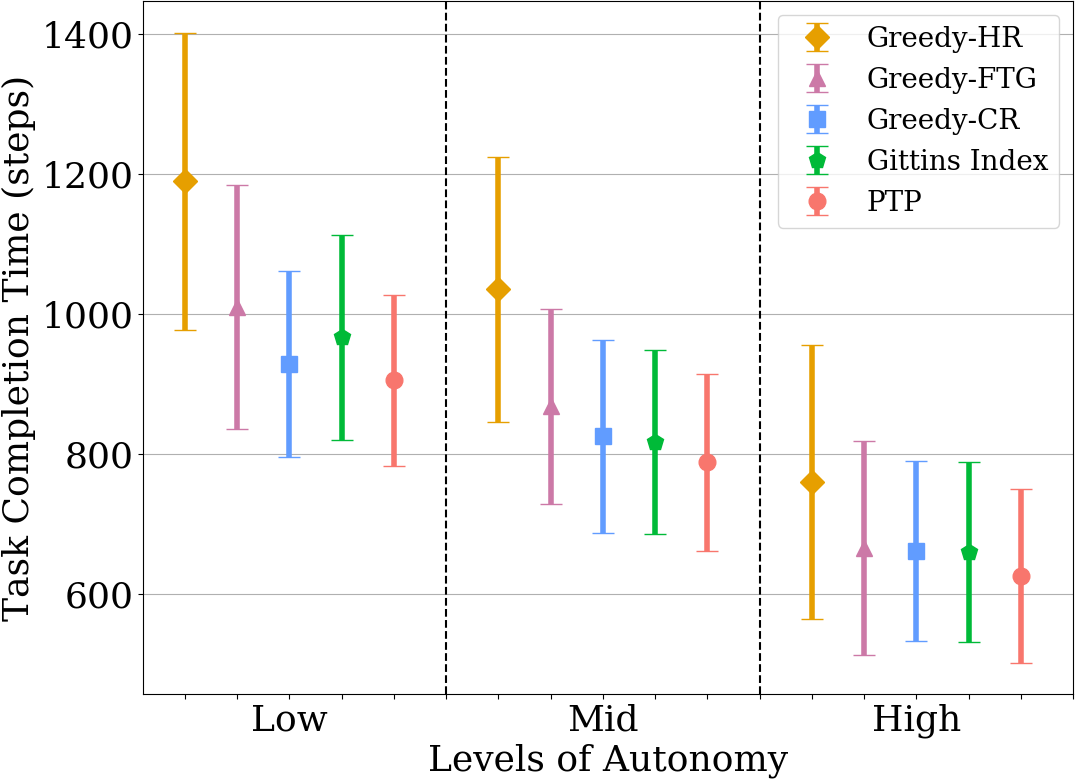}
  \caption{\textbf{Task completion time with varying levels of autonomy.} The robot fleet size is fixed to mid.}
  \label{fig:task-completion-time-fixed-fleet-size}
  \vspace{-2mm}
\end{figure}

\subsection{Results}
\label{subsec:results}
In the first experiment, we fix the robot fleet size to mid and vary the levels of autonomy. The results are presented in Table~\ref{table:results-mid-fleet-size} in the order of \textit{low}/\textit{mid}/\textit{high} autonomy. The task completion time with averages and standard deviations over different fields and trials are visualized in Figure~\ref{fig:task-completion-time-fixed-fleet-size}. As shown, all the methods are able to complete the task in less time with the increase of levels of autonomy as the robots encounter fewer failures. By performing long-horizon planning for the supervisor on the graph, PTP achieves the fastest task completion time and human working time among the five methods under all levels of autonomy. The visualization of human working time reveals similar patterns to those from task completion time and thus omitted for brevity. The observation of longer human working time under lower autonomy also supports the intuition that the supervisor is required to provide more assistance to the robots with more failures occurred.

In the second experiment, we fix the levels of autonomy to low, which generates the smallest margin between PTP and other baselines in the first experiment, and vary the robot fleet size. The results are presented in Table~\ref{table:results-low-autonomy} in the order of \textit{small}/\textit{mid}/\textit{large} fleet size. Similarly, we visualize the task completion time with averages and standard deviations over different fields and trials in Figure~\ref{fig:task-completion-time-fixed-autonomy}. Similar to the findings in the first experiment, PTP outperforms the compared methods in both metrics under most conditions except for a longer human working time with small robot fleet size. As a side note, we observe that the navigation task can be completed in less time with more robots in the team for all methods.

\begin{table}[t]
  \begin{center}
    \caption{Results of fleet management with varying robot fleet size}
    \label{table:results-low-autonomy}
    \begin{tabular}{ l | c  c }
      \toprule
      Methods & Task Completion Time & Human Working Time \\
      \midrule
      \rule{-2.5pt}{2ex} Greedy-HR & $1272.0 \, / \, 1190.3 \, / \, 1049.8$ & $1199.0 \, / \, 1167.0 \, / \, 1037.5$ \\
      Greedy-FTG & $1096.6 \, / \, 1010.4 \, / \, 987.1$ & $1074.6 \, / \, 997.7 \, / \, 978.7$ \\
      Greedy-CR & $1060.0 \, / \, 929.3 \, / \, 858.0$ & $\mathbf{1011.6} \, / \, 900.7 \, / \, 841.4$ \\
      Gittin's Index & $1077.4 \, / \, 967.0 \, / \, 849.1$ & $1021.8 \, / \, 938.3 \, / \, 834.0$ \\
      PTP & $\mathbf{1042.6} \, / \, \mathbf{905.6} \, / \, \mathbf{815.6}$ & $1017.4 \, / \, \mathbf{892.9} \, / \, \mathbf{806.4}$ \\
      \bottomrule
    \end{tabular}
  \end{center}
\end{table}

\begin{figure}[t]
  \centering
  \includegraphics[width=0.9\linewidth]{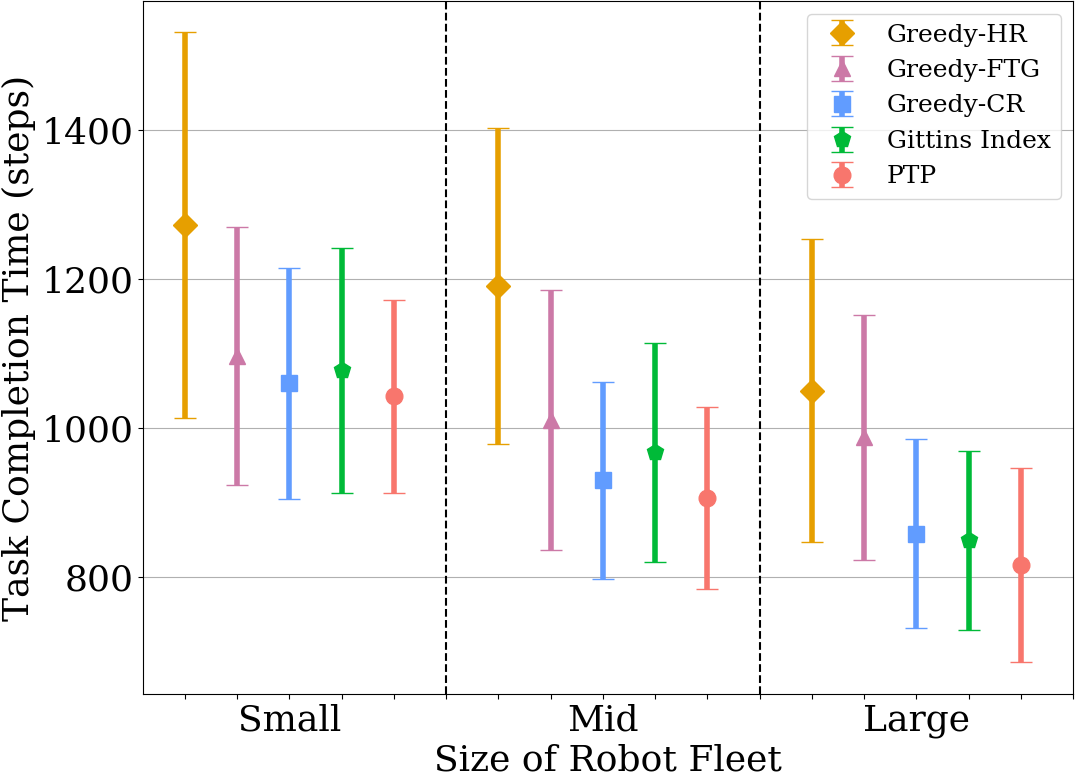}
  \caption{\textbf{Task completion time with varying robot fleet size.} The levels of autonomy is fixed to low.}
  \label{fig:task-completion-time-fixed-autonomy}
  \vspace{-3mm}
\end{figure}

To better understand the difference between the effects of the actions planned by the PTP and other baselines, we examine the progress made by the human-robot team against time. The percentage of field covered over time with mid autonomy and mid robot fleet size is plotted in Figure~\ref{fig:covered-field-vs-time}. As shown in Figure~\ref{subfig:covered-field-vs-time-global}, Greedy-CR and Gittins index exhibit better performance than PTP in the first half of the simulation as both baselines encourage the supervisor to rescue the close robots. With shorter robot waiting time, the robot fleet can make faster progress at the initial stage of the navigation task. In contrast, the PTP provides long-term planning for the supervisor and considers robot coordination through the modified cost~(\ref{eq:modified-cost}), thus generating superior global performance on the task compared with other baselines as shown in Figure~\ref{subfig:covered-field-vs-time-local} and the second half of Figure~\ref{subfig:covered-field-vs-time-global}. In other words, each robot is encouraged to wait if the supervisor can benefit from rescuing another robot that either needs more urgent help or can save the distance traveled by the supervisor in the future.

\begin{figure}[t]
  \centering
  \begin{subfigure}[b]{0.49\linewidth}
    \includegraphics[width=\linewidth]{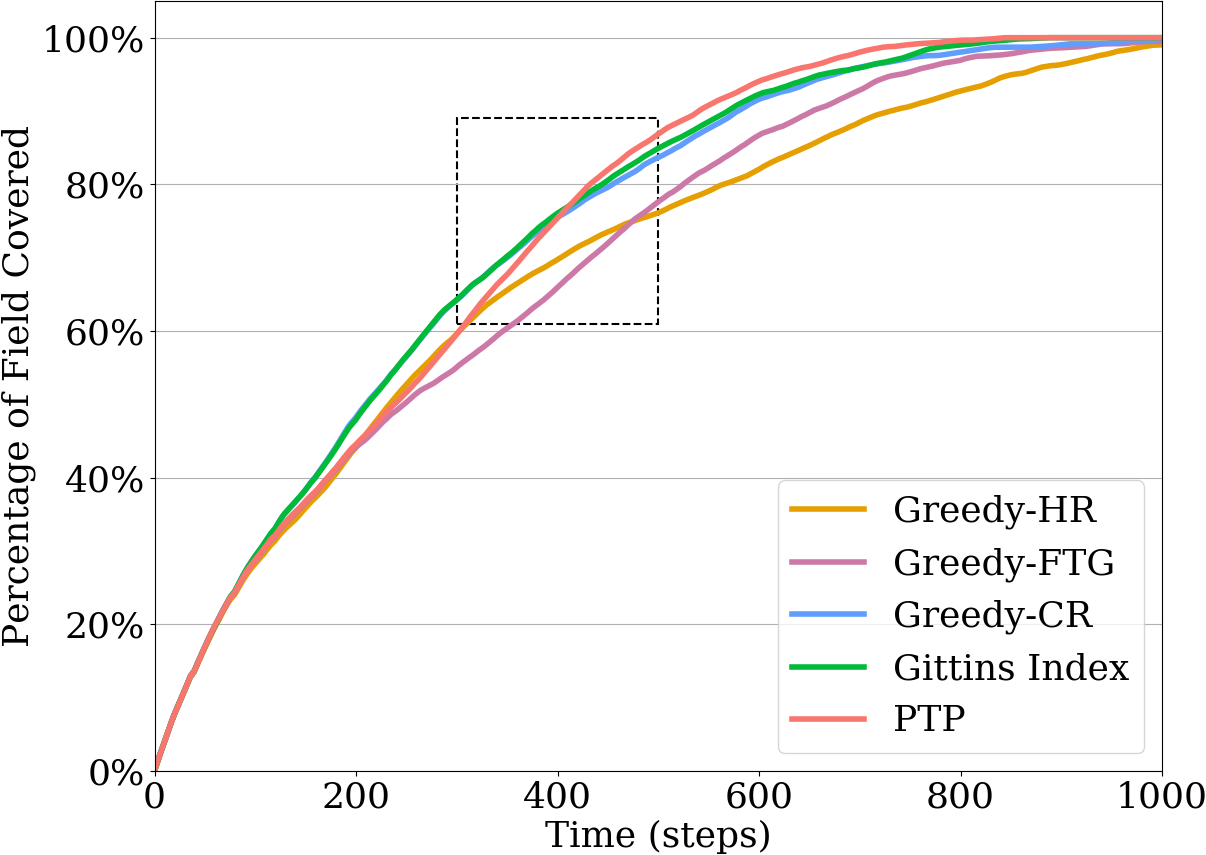}
    \caption{}
    \label{subfig:covered-field-vs-time-global}
  \end{subfigure}
  \begin{subfigure}[b]{0.49\linewidth}
    \includegraphics[width=\linewidth]{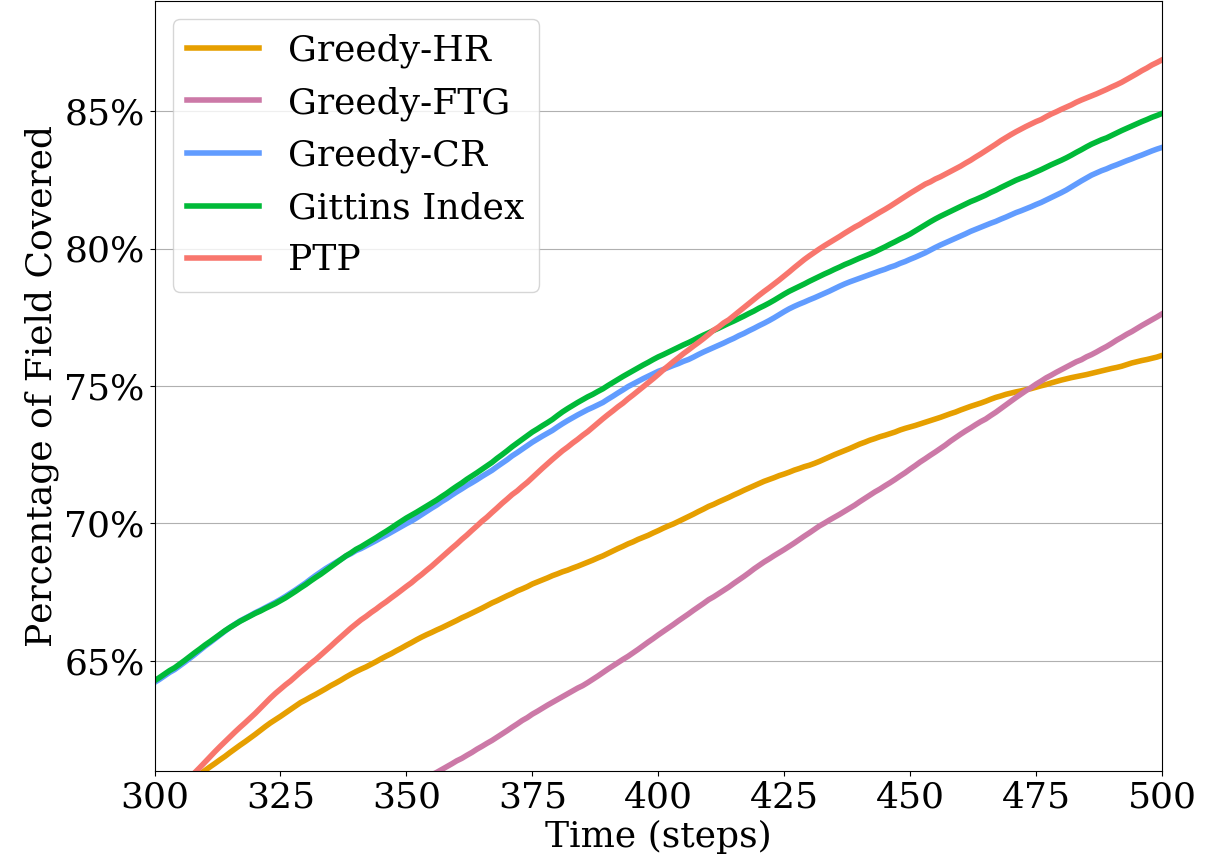}
    \caption{}
    \label{subfig:covered-field-vs-time-local}
  \end{subfigure}
  \caption{\textbf{Percentage of field covered over time.} The standard deviations over different trials are omitted for clear presentation. \textbf{(a)} The average percentage of field covered over time across the whole simulation. \textbf{(b)} The details of figure around $405$ time steps when PTP surpasses the second best method in terms of task completion progress.}
  \label{fig:covered-field-vs-time}
  \vspace{-3mm}
\end{figure}
\section{Discussion and Limitations}
The results presented in Section~\ref{sec:case-study} show that compared to baseline methods, the modified PTP provides a promising solution that can enable the human-robot team to complete collaborative tasks with higher efficiency. Despite the advantages, our work also encompasses several limitations, which will be discussed in detail in this section.

One of the main challenges in problem formulation of fleet management is the definition of the reward function $r_i(t)$ for each robot. In this work, the reward is defined as the \textit{net} expected distance that a robot can travel with supervisor intervention, which can change abruptly from $0$ to a positive number at time $t$ when the robot fails. Such discontinuity in time can result in bad performance of the solution obtained at time $t-1$ on the updated dynamic GTP, which can also be seen from the bound~(\ref{ineq:theorem}) as $\alpha$ needs to be large enough to encompass the change in one time step. To resolve the issue, a smooth reward function over time is desired to represent the ``value" of visiting a robot.

In addition, as noted in Section~\ref{subsec:results}, we observe that the modified PTP generates a longer robot waiting time compared to Greedy-CR and Gittins Index, i.e., the robots spend more time on waiting for help from human after a failure occurs. Such a low robot utilization can increase the energy cost in practice when the robots are idle but can be potentially improved by formulating and incorporating the robot waiting time into the objective function of the modified PTP.

Apart from the performance, the real-time capability is another important factor that affects the practical value of an algorithm in real-world applications. To act intelligently, the supervisor in the field requires prompt guidance from the system with small time delays. In our case study, the modified PTP returns the optimal solution within $2.26$, $4.96$, and $24.02$ ms on average in a complete graph with $6$, $8$, and $11$ nodes. Although the algorithm is efficient enough to support the management of robot fleets with moderate size, it is known that the worst-case running time of exact algorithms for PTPs scales superpolynomially with the number of nodes in a graph~\cite{bienstock1993note}. With the recently presented idea of learning heuristics for combinatorial optimization problems~\cite{kool2018attention}, reinforcement learning approaches can be explored to approximate the optimal solution to the dynamic GTP.

Despite the limitations and potential improvements, we hope that this work builds a rigorous framework for fleet management problems with physical assistance, provides a competitive algorithm that enables efficient human-robot collaboration, and can encourage the research on one-to-many supervision from a new perspective of graph traversal problem.
\section{Conclusion}
In this work, we formulate the one-to-many supervision of multi-robot systems as a dynamic graph traversal problem, in which the node rewards represent the value of visiting a robot and the edge costs reflect the traveling cost from one node to another. We approximate the solution to the intractable dynamic problem by solving a static GTP and develop a bound on the approximation error. A practical implementation on optimizing the static GTP is developed based on a modified profitable tour problem, whose optimal solution can be obtained efficiently using mixed-integer programming. Our case study shows that the proposed method outperforms various baselines with faster task completion time and human working time under various conditions in a simulated autonomous farm. Equipped with such an intelligent advising system, the performance of the human-multi-robot team can be potentially enhanced in real-world applications where physical assistance is required. With several possible directions to explore in the future, we hope that this work marks the first step towards formalizing a general supervision problem as a graph traversal problem and can serve as a benchmark for future methods.

\section*{Acknowledgements}
This work was supported by the the USDA National Institute of Food and Agriculture (USDA/NIFA), through the National Robotics Initiative 2.0 (NIFA\#2021-67021-33449), the AI Institute AIFARMS through the Agriculture and Food Research Initiative (AFRI) (USDA/NIFA Award no. 2020-67021-32799), as well as the Illinois Center for Digital Agriculture.



\bibliographystyle{plainnat}
\bibliography{references}

\end{document}